\documentclass[sigconf, nonacm]{acmart}
\AtBeginDocument{%
  }

\usepackage{algorithm}
\usepackage{algorithmic}
\usepackage{makecell}
\usepackage{booktabs}

\usepackage[utf8]{inputenc}
\usepackage{amsmath}
\usepackage{amssymb}
\usepackage{array}
\usepackage{tabularx}
\usepackage{geometry}

\usepackage{multirow}   

\begin{document}

\title{ClockRoPE: Random Fourier Rotations for Temporal Routine Modeling}

\author{Yiwen Chen}
\affiliation{%
  \institution{YouTube}
  \country{}  
}
\email{ywenc@google.com}

\author{Joshua Ainslie}
\affiliation{%
 \institution{Google DeepMind}
  \country{}  
}
 \email{jainslie@google.com}

 \author{Krzysztof Choromanski}
\affiliation{%
  \institution{Google DeepMind}
  \country{}  
}
\email{kchoro@google.com}

\author{Xiang Gao}
\affiliation{%
  \institution{YouTube}
   \country{}  
}
\email{xggao@google.com}

\author{Su-Lin Wu}
\affiliation{%
  \institution{YouTube}
   \country{}  
}
\email{sulin@google.com}

\author{Yiping Yuan}
\affiliation{%
  \institution{YouTube}
   \country{}  
}
\email{yipingyuan@google.com}

\author{Qian Sun}
\affiliation{%
  \institution{YouTube}
  \country{}  
}
\email{qians@google.com}

\begin{abstract}
Rotary Position Embedding (RoPE) has been widely
adopted in transformer-based large language models. However, its log-linear frequency schedule, originally designed to produce long-term attention decay, limits its adoption in domains with more complex distance-correlation patterns, such as temporal periodicity in sequential recommendation. We investigate the expressiveness of general query/key rotations and find that any normalized continuous positive-definite attention modulation function can be approximated by random rotations induced by its own Fourier transform, which we term Random Fourier Rotations. Building on this theory, we propose ClockRoPE for routine modeling in sequential recommendation, where rotation frequencies are derived from periodic attention modulation functions.  In online A/B tests, ClockRoPE demonstrates consistent improvements in valued engagement metrics, and has been successfully deployed in production-scale generative retrieval system at a major video-sharing platform.
\end{abstract}







\maketitle

\section{Introduction}

As transformers\cite{vaswani2017attention} and the scaling law \cite{kaplan2020scaling, hoffmann2022training} have demonstrated great success in language modeling \cite{devlin2018bert, brown2020language} and image generation \cite{peebles2023scalable}, the generative recommendation field is converging towards transformer-style architectures for sequence modeling \cite{kang2018self, zhai2024actions}. However different from language modeling, time is a crucial and native signal in sequential recommendation. Various methods have been explored to incorporate temporal signals in transformer-based sequential recommendation, including  time difference input feature \cite{chai2025longer}, additive temporal attention bias \cite{ zhai2024actions} and relative time embedding within attention \cite{li2020time}.

On the other hand, RoPE \cite{su2021roformer} has shown superior performance in transformer-based language models, featuring properties such as translation-invariance and long-term attention decay. Despite being originally designed for 1D discrete token positions, RoPE has quickly seen adoption in domains with multi-dimensional continuous coordinates \cite{li2026nd,string-paper,ropes-for-vits}.

 Naturally, several recent work have argued for extending RoPE to encode time in sequential recommendation \cite{wei2025rotate, hou2026kunlun}: \cite{wei2025rotate} adopts similar log-linear frequency spacing as in standard RoPE and replaces position with timestamp delta from an anchor timestamp; ROTE \cite{hou2026kunlun} first computes log-scaled timestamp gaps and constructs rotation angles based on them to capture temporal proximity. Both work emphasize on using timestamp together with order to encode recency.

However, an important characteristic of user behavior in recommender systems is periodicity. 
Interactions occurring at the same hour of the day or the same day of the week often exhibit high correlations. Yet, the log-linear frequency schedule in standard RoPE produces long-term attention decay and struggles to capture more complex distance-correlation patterns including periodicity.

To address this limitation, we leverage Fourier analysis to investigate the expressiveness of general query/key rotations. We prove that any continuous normalized positive-definite modulation function can be unbiasedly approximated by sampling RoPE frequencies from their own Fourier transform distribution.  
In summary, the contributions of our work include:

\begin{itemize}
    \item  \textbf{Random Fourier Rotations (RFR):} a general sampling method for approximating arbitrary normalized positive-definite attention modulation functions using query/key rotations.
    \item \textbf{ClockRoPE:} A temporal periodicity encoding approach for user routine modeling, derived from applying Random Fourier Rotations to approximate periodic attention modulation functions.
    \item {\textbf{Production-Scale Evaluation:}} Empirical evidence of ClockRoPE's effectiveness in improving model's routine-awareness and overall viewer value within a production-scale generative
retrieval system. 
\end{itemize}

\section{Related Work}

\subsection{Relative Positional Encoding}
\textit{Relative Positional Encoding} (RPE) techniques \citep{shaw-rpe, HuangVUSHSDHDE19, fire, block-toeplitz, raffel} is a powerful class of positional encoding methods used in Transformer models. In contrast to \textit{absolute positional encodings} (APEs) \citep{vaswani2017attention, kazemi2019time2vec} that add / concatenate sinusoidal-based embeddings of the tokens' absolute positions with content embeddings, they more explicitly modulate attention matrix values via (learnable) functions of the relative positions / distances between the tokens in the corresponding metric space (e.g. 1D line as it is the case for language, 2D grids for images or 3D grids for videos). RPEs is a rich class of techniques and is often taxonomized via two branches: (1) the \textit{additive mechanisms} \citep{Fourier-RPE, block-toeplitz, fire} that introduce additive terms in the logits space (numerically equivalent to Hadamard multiplication obtained after element-wise logits exponentiation) and (2) the \textit{multiplicative mechanisms} \citep{su2021roformer, ropes-for-vits, string-paper} that modulate query-key dot-products by introducing multiplicative terms (those two strategies can also be combined into hybrid methods). RPEs are used across different modalities: language \citep{su2021roformer}, images \citep{block-toeplitz, ropes-for-vits, string-paper} and more recently, even point clouds \citep{relflexformer}.

RPEs can be often efficiently implemented, in time near-linear in the input sequence length (rather than quadratic). For instance, it is well known \citep{block-toeplitz, fast-rpe-base} that as long as the mask matrix $\mathbf{M}$ used in the Hadamard multiplication of the additive RPE mechanism supports sub-quadratic (in sequence length) matrix-vector multiplication, the corresponding RPE mechanism can be applied to linear low-rank attention Transformers \citep{performers} in sub-quadratic time.
Similarly, multiplicative RPE mechanisms often support efficient implementations (see: Sec. \ref{sec:related_work_rope}).


\subsection{RoPE \& STRING}
\label{sec:related_work_rope}
\textit{Rotary Position Encodings}, or: RoPE \cite{su2021roformer, ropes-for-vits, ropes-for-graphs, li2026nd, hua2024fourier}, is an instantiation of the multiplicative RPE mechanism, where the queries and keys are post-processed by rotations matrices that depend on the positions of their corresponding tokens. More specifically, a rotation matrix corresponding to the token is defined as a product of disjoint two-dimensional (\textit{Givens}) rotations with rotations angles defined as linear functions of tokens's positions (1D or higher-dimensional) with potentially trainable parameters (different variations include axial-RoPe, mixed-RoPE and more). This particular design choice for the rotation matrices corresponding to the tokens is a gateway to re-writing the modulated attention scores as: $\mathbf{q}^{\top}\mathbf{R}(\delta \mathbf{r})\mathbf{k}$, where $\mathbf{R}(\delta \mathbf{r})$ is a rotation matrix defined as a product of disjoint $2$-dimensional rotations with angles given as linear functions of the relative positions. 

The rigid structure of the rotations matrices used to rotate queries and keys in the RoPE mechanism leads to \textit{translation invariance} (the mechanism depends only on the relative positions), but in practice can be significantly relaxed without sacrificing this property. The recently introduced STRING method \citep{string-paper} is a super-set of RoPEs, where the rotation matrix is obtained by exponentiating (in the matrix-exponentiation rather than element-wise sense) \textit{skew-symmetric} (anti-symmetric) matrices. Each skew-symmetric matrix is defined as a linear combination of the skew-symmetric \textit{generators} (often learnable) with coefficients given by the coordinates of the tokens corresponding to queries/keys. As it is shown in \citep{string-paper}, as long as the generators commute, the resulting mechanism remains translation-invariant, with RoPE providing its special instantiation. Furthermore, it is the most general extension of the RoPE mechanism in the setting, where queries and keys are post-processed via rotation matrices and that keeps translation-invariance.

Both STRING and RoPE trivially support fast computations with efficient attention methods, since they do not require the explicit materialization of the logits matrix, acting independently on the set of queries and keys.

\subsection{Fourier Methods in Sequence Modeling}

Fourier methods are a common tool in sequence modeling. One early example for Transformer models in particular is FNet \cite{leethorp2021fnet}, which replaced the self-attention mechanism in Transformer layers with the discrete Fourier transform. Despite lacking learnable parameters, the discrete Fourier transform served as an efficient token "mixing" mechanism to replace attention while retaining much of the quality in BERT \cite{devlin2018bert} models.

Fourier analysis has also been used to improve RoPE itself. FoPE
\cite{hua2024fourier} leverages Discrete Signal Processing theory to improve length
generalization, while nD-RoPE \cite{li2026nd} derives a spectral isotropy
condition to generalize RoPE to higher-dimensional spaces. Both refine
RoPE's existing frequency structure, whereas ClockRoPE derives a new frequency-sampling distribution targeting periodic attention modulation.

\subsection{Fourier Features}

Fourier features have a broad history in kernel approximation and
function representation. Random Fourier Features (RFF) \cite{rahimi2007random} approximate shift-invariant
kernels via random sinusoidal projections sampled from the kernel's
Fourier transform. These features have since been shown to help networks learn
high-frequency functions~\cite{tancik2020fourier}, encode continuous
time in self-attention~\cite{xu2019self}, and support learnable
frequency bases for spatial positional
encoding~\cite{li2021learnable}. Unlike these methods,
which construct Fourier-based input features, our Random Fourier Rotations builds on RFF and extend from feature maps to rotation operators on queries and keys, preserving RoPE's multiplicative
structure and efficient implementation.

\section{Methodology}

\subsection{Motivation}
\label{sec:motivation}
\begin{figure}
    \centering
    \includegraphics[width=1.0\linewidth]{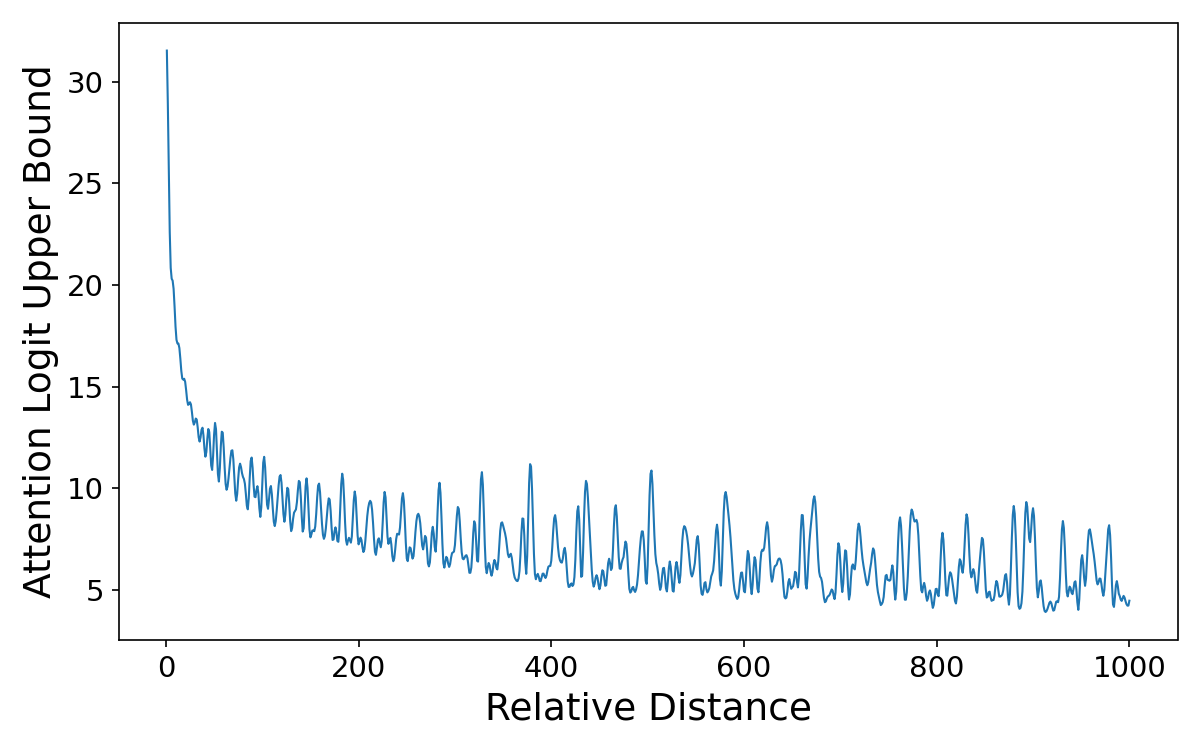}
    \caption{Long-term decay of RoPE}
    \label{fig:decay}
\end{figure}

RoPE  \cite{su2021roformer}  is built with 2d rotation per feature pair, which is inherently periodic. However as proved in the original paper \cite{su2021roformer} , when combining these 2d rotations with the standard log-linear spaced frequency schedule : $\theta_i = 10000^{-2i/d}, \quad i \in [0, 2, \ldots, d/2-1]$,  RoPE enforces attention score upper bound decay with relative distance, exhibiting a monotonically decreasing trend with random oscillations as shown in Figure \ref{fig:decay}. This property is termed as long-term decay in the original paper. Although the long-term decay property suits the nature of language modeling, it risks oversimplifying how distances affects correlation in domains with natural periodicity, such as generative recommendation. We empirically confirm this effect in \S\ref{sec:attn} Figure \ref{fig:attn}, where standard RoPE is shown to concentrate attention on recent interactions while obscuring the daily periodicity pattern.

 In sequential generative recommendation, interactions at similar hour of the day or day of the week exhibit high correlations. An ideal frequency schedule for periodicity modeling provides higher attention score upper bound for interaction pairs at similar hour-of-the-day or day-of-the-week than others.

\subsection{Modulating Attention Logits via Random Fourier Rotations}

Before jumping into our specific approach for periodic attention modulation, we first
introduce the theoretic foundation for modulating attention scores based on continuous token positions.

Suppose we want to modulate the standard attention logits such
that the similarity between a query $\boldsymbol{q}_m$ at 
position $p_m$ and a key $\boldsymbol{k}_n$ at position $p_n$ 
is weighted by a positive definite 
kernel $f: \mathbb{R} \to \mathbb{R}$  with $f(0) = 1$:

\begin{equation}
\boldsymbol{q}_m^\top \boldsymbol{k}_n \leadsto \boldsymbol{q}_m^\top \boldsymbol{k}_n f(p_m - p_n) 
\end{equation}
where $p_m, p_n \in \mathbb{R}$. 
For example, $f$ could be Gaussian: $f(x) = e^{-x^2}$,  Laplace: $f(x) = e^{-|x|}$ or trigonometric $f(x) = \cos (x)$.

We claim in the following proposition that any normalized continuous positive-definite modulation function can be unbiasedly approximated by sampling RoPE frequencies from their own Fourier transform distribution.
\begin{proposition}[Random Fourier Rotation Estimator]
Let $\xi_0, \xi_1, \ldots, \xi_{d/2-1} \overset{\mathrm{i.i.d.}}{\sim} 
\tau(\xi)$, where $\tau(\xi)$ is the Fourier transform of a continuous
positive definite kernel $f: \mathbb{R} \to \mathbb{R}$  with $f(0) = 1$ \footnote{ $f(0)=1$ is required to satisfy the normalized condition in Bochner's theorem ~\cite{bochner1933monotone}}. Let $\boldsymbol{q}_m^{(j)} = 
[\boldsymbol{q}_{m,2j}, \boldsymbol{q}_{m,2j+1}]^\top \in 
\mathbb{R}^2$ and $\boldsymbol{k}_n^{(j)} = [\boldsymbol{k}_{n,2j}, 
\boldsymbol{k}_{n,2j+1}]^\top \in \mathbb{R}^2$ denote the 
$j$-th consecutive feature pair of query $\boldsymbol{q}_m \in 
\mathbb{R}^d$ and key $\boldsymbol{k}_n \in \mathbb{R}^d$ 
respectively, for $j = 0, \ldots, d/2-1$. Then the Random 
Fourier Rotation estimator:
\begin{equation}
\hat{g}(\boldsymbol{q}_m, \boldsymbol{k}_n, p_m, p_n) = \sum_{j=0}^{d/2-1} (\boldsymbol{R}_m^{(j)} 
\boldsymbol{q}_m^{(j)})^\top (\boldsymbol{R}_n^{(j)} 
\boldsymbol{k}_n^{(j)})
\end{equation}

where 

\begin{equation}
\boldsymbol{R}_m^{(j)} = \begin{bmatrix} \cos(2\pi\xi_j p_m) & 
-\sin(2\pi\xi_j p_m) \\ \sin(2\pi\xi_j p_m) & 
\cos(2\pi\xi_j p_m) \end{bmatrix}
\end{equation}

\begin{equation}
\boldsymbol{R}_n^{(j)} = \begin{bmatrix} \cos(2\pi\xi_j p_n) & 
-\sin(2\pi\xi_j p_n) \\ \sin(2\pi\xi_j p_n) & 
\cos(2\pi\xi_j p_n) \end{bmatrix}
\end{equation}

is an unbiased estimator of $\boldsymbol{q}_m^\top 
\boldsymbol{k}_n f(p_m - p_n)$, i.e.:
\begin{equation}
\mathbb{E}_{\xi_0, \ldots, \xi_{d/2-1} \sim \tau(\xi)}
\left[\sum_{j=0}^{d/2-1} (\boldsymbol{R}_m^{(j)} 
\boldsymbol{q}_m^{(j)})^\top (\boldsymbol{R}_n^{(j)} 
\boldsymbol{k}_n^{(j)})\right] = \boldsymbol{q}_m^\top 
\boldsymbol{k}_n f(p_m - p_n)
\end{equation}
\end{proposition}

\begin{proof}
By Bochner's theorem~\cite{bochner1933monotone}, since $f$ is continuous
positive definite with $f(0) = 1$, 
its Fourier transform $\tau(\xi) = \int_{\mathbb{R}} f(x) 
e^{-i2\pi\xi x} dx$ is a proper probability density function 
satisfying $\tau(\xi) \geq 0$ and $\int_{\mathbb{R}} \tau(\xi) 
d\xi = f(0) = 1$. This allows us to express $f$ as an 
expectation over $\xi$ similar to the derivation in 
RFF~\cite{rahimi2007random}:

\begin{align}
f(p_m - p_n) &= \int_{\mathbb{R}} e^{i2\pi\xi(p_m-p_n)} 
\tau(\xi) d\xi \notag \\
&= \mathbb{E}_{\xi \sim \tau(\xi)}\left[e^{i2\pi\xi (p_m - p_n)} \right]
\end{align}

Grouping query and key features into $d/2$ pairs 
$\boldsymbol{q}_m^{(j)} = [\boldsymbol{q}_{m,2j}, 
\boldsymbol{q}_{m,2j+1}]^\top$ and $\boldsymbol{k}_n^{(j)} = 
[\boldsymbol{k}_{n,2j}, \boldsymbol{k}_{n,2j+1}]^\top$, we have:

\begin{equation}
\boldsymbol{q}_m^\top \boldsymbol{k}_n f(p_m - p_n) = 
\sum_{j=0}^{d/2-1} \left(\boldsymbol{q}_m^{(j)}\right)^\top 
\boldsymbol{k}_n^{(j)} f(p_m - p_n)
\end{equation}

Using the complex representation of 2d vector $\boldsymbol{q}_m^{(j)} \mapsto 
\boldsymbol{q}_{m,2j} + i\boldsymbol{q}_{m,2j+1}$, 
$\boldsymbol{k}_n^{(j)} \mapsto \boldsymbol{k}_{n,2j} + 
i\boldsymbol{k}_{n,2j+1}$, and the property 

\begin{equation}
\label{eq:dot}
\boldsymbol{q}_m^{(j)\top}\boldsymbol{k}_n^{(j)} = 
\mathrm{Re}[\boldsymbol{q}_m^{(j)}{\boldsymbol{k}_n^{(j)}}^*]
\end{equation}, where ${\boldsymbol{k}_n^{(j)}}^*$ is the complex conjugate of ${\boldsymbol{k}_n^{(j)}}$,
we arrive at:
\begin{align}
\boldsymbol{q}_m^\top \boldsymbol{k}_n f(p_m - p_n) 
&= \mathrm{Re}\left[\sum_{j=0}^{d/2-1} \boldsymbol{q}_m^{(j)} 
{\boldsymbol{k}_n^{(j)}}^* f(p_m - p_n)\right] \notag \\
&= \mathrm{Re}\left[\sum_{j=0}^{d/2-1} \boldsymbol{q}_m^{(j)} 
{\boldsymbol{k}_n^{(j)}}^* \mathbb{E}_{\xi \sim \tau(\xi)}
\left[e^{i2\pi\xi (p_m - p_n)}\right]\right] 
\end{align}

On the other hand, 

\begin{align}
&\mathbb{E}_{\xi_0,\ldots,\xi_{d/2-1} \sim \tau(\xi)}
\left[\sum_{j=0}^{d/2-1} (\boldsymbol{R}_m^{(j)} 
\boldsymbol{q}_m^{(j)})^\top (\boldsymbol{R}_n^{(j)} 
\boldsymbol{k}_n^{(j)})\right] \notag \\
&= \mathbb{E}_{\xi_0,\ldots,\xi_{d/2-1} \sim \tau(\xi)}
\left[\sum_{j=0}^{d/2-1} \mathrm{Re}\left[\left(
\boldsymbol{q}_m^{(j)} e^{i2\pi\xi_j p_m}\right)
\left(\boldsymbol{k}_n^{(j)} e^{i2\pi\xi_j p_n}
\right)^*\right]\right] \notag \\
&= \mathbb{E}_{\xi_0,\ldots,\xi_{d/2-1} \sim \tau(\xi)} \mathrm{Re}\left[\sum_{j=0}^{d/2-1}
\boldsymbol{q}_m^{(j)} {\boldsymbol{k}_n^{(j)}}^* e^{i2\pi\xi_j p_m}
 e^{-i2\pi\xi_j p_n}\right] \\
 &= \mathrm{Re}\left[\sum_{j=0}^{d/2-1}
\boldsymbol{q}_m^{(j)} {\boldsymbol{k}_n^{(j)}}^* \mathbb{E}_{\xi \sim \tau(\xi)} \left[e^{i2\pi\xi_j p_m}
 e^{-i2\pi\xi_j p_n} \right]\right] \\
&= \boldsymbol{q}_m^\top \boldsymbol{k}_n f(p_m - p_n) 
\end{align}

which completes the proof.
\end{proof}

Next, we investigate how fast Random Fourier Rotations converges as the number of feature pairs increases . 

\begin{proposition}[Convergence of Random Fourier Rotation Estimator]
Under the same setting as Proposition 3.1, for any $\epsilon > 0$:

\begin{align}
& P\left(\left| \frac{1}{d/2} \sum_{j=0}^{d/2-1} (\boldsymbol{R}_m^{(j)} 
\boldsymbol{q}_m^{(j)})^\top (\boldsymbol{R}_n^{(j)} 
\boldsymbol{k}_n^{(j)}) - \frac{1}{d/2}  \boldsymbol{q}_m^\top 
\boldsymbol{k}_n f(p_m - p_n) \right| \geq \epsilon
\right) \notag \\
& \leq 2\exp\left(-\frac{\epsilon^2 d^2}{8\sum_{j=0}^{d/2-1} (\left\|\boldsymbol{q}_m^{(j)}\right\| \left\|\boldsymbol{k}_n^{(j)}\right\| )^2}
\right)
\end{align}
\end{proposition}

The proof of Proposition 3.2 can be found in Appendix A.1

In particular, when $\sum_{j=0}^{d/2-1}((\boldsymbol{q}_m^{(j)})^\top
\boldsymbol{k}_n^{(j)})^2$ scales linearly with $d$, e.g. at the beginning of training,
Proposition~3.2 guarantees convergence of
\[
\frac{1}{d/2}\sum_{j=0}^{d/2-1}
(\boldsymbol{R}_m^{(j)}\boldsymbol{q}_m^{(j)})^\top
(\boldsymbol{R}_n^{(j)}\boldsymbol{k}_n^{(j)})
\]
to
\[
\frac{1}{d/2}\boldsymbol{q}_m^\top\boldsymbol{k}_n
f(p_m-p_n)
\] exponentially fast in $d$, 
enforcing the attention modulation prior $f$ from the start of training. See Appendix A.2 for more details

The Random Fourier Rotations algorithm is summarized in Algorithm \ref{alg:rfr}.

\subsubsection{Random Fourier Rotations for Periodic Functions}

We highlight a special case for periodic functions here as their Fourier transform result in discrete fourier series, forming a probability mass function instead of a density function in the general case. 

\begin{corollary}[Periodic Case via Herglotz's Theorem]
Let $f$ be a periodic continuous positive definite
kernel with period $T$, $f(0) = 1$.
The complex Fourier series of $f$ is
$$
\sum_{k = -\infty}^\infty \alpha_k e ^{i 2\pi k x / T}
$$

where
\begin{equation}
\alpha_k = \frac{1}{T} \int_0^T f(x) e^{- i 2\pi k x / T} \, dx = \frac{1}{T} \int_0^T f(x) \cos(2\pi kx/T)\, dx
\end{equation}

By Herglotz's theorem~\cite{herglotz1911uber},
the Fourier coefficients
$\{\alpha_k\}_{k=-\infty}^{\infty}$
satisfy $\alpha_k \geq 0$ and
$\sum_{k= -\infty}^\infty \alpha_k = f(0) = 1$, forming a valid
probability mass function over the discrete harmonics
$\{\xi_k = k/T\}_{k=-\infty}^\infty$. Consequently,
Propositions~3.1 and~3.2 apply directly with
frequencies sampled from $\{\xi_k = k/T\}_{k=-\infty}^\infty$
with PMF $P(\xi) = \sum_{k=-\infty}^\infty \alpha_k \delta(\xi - k/T)$.
\end{corollary}

\begin{algorithm}
\caption{Random Fourier Rotations for Attention Modulation}\label{alg:rfr}
\begin{algorithmic}[1]
\REQUIRE A continuous positive definite attention modulating prior
$f$ with $f(0) = 1$, embedding 
dimension $d$, continuous positions $p_m, p_n \in \mathbb{R}$
\ENSURE Rotation matrice $\boldsymbol{R}_m^d$ and $\boldsymbol{R}_n^d$ such that 
$(\boldsymbol{R}_m^d \boldsymbol{q}_m)^\top
(\boldsymbol{R}_n^d \boldsymbol{k}_n) \approx 
\boldsymbol{q}_m^\top \boldsymbol{k}_n f(p_m - p_n)$

\STATE Compute the Fourier transform of $f$: 
$$\tau(\xi) = \int_{\mathbb{R}} f(x) e^{-i2\pi\xi x} dx$$
\STATE Draw $d/2$ i.i.d. samples $\xi_0, \xi_1, \ldots, 
\xi_{d/2-1} \sim \tau(\xi)$
\STATE For each feature pair $j = 0, \ldots, d/2-1$, 
construct rotation matrices:
$$\boldsymbol{R}_m^{(j)} = 
\begin{bmatrix} 
\cos(2\pi\xi_j p_m) & -\sin(2\pi\xi_j p_m) \\ 
\sin(2\pi\xi_j p_m) & \cos(2\pi\xi_j p_m) 
\end{bmatrix}$$

$$
\boldsymbol{R}_n^{(j)} = \begin{bmatrix} \cos(2\pi\xi_j p_n) & 
-\sin(2\pi\xi_j p_n) \\ \sin(2\pi\xi_j p_n) & 
\cos(2\pi\xi_j p_n) \end{bmatrix}
$$

stack $\{\boldsymbol{R}_m^{(j)}\}_{j=0}^{d/2-1}$ and $\{\boldsymbol{R}_n^{(j)}\}_{j=0}^{d/2-1}$ in a block diagonal fashion to form $\boldsymbol{R}_m^d$ and $\boldsymbol{R}_n^d$ respectively.
\end{algorithmic}
\end{algorithm}

\subsection{ClockRoPE: Application in Capturing User Routines}

\begin{table*}[htbp]
  \centering
  \caption{ClockRoPE Sampling Distributions}
  \label{tab:dist}
  \renewcommand{\arraystretch}{1.2} 

  \begin{tabular}{lcl} 
    \toprule
    \textbf{Distributions} & \textbf{Form} & \textbf{Parameters} \\
    \midrule
    cosine-sym & $ \frac{1}{2} \delta\left(\xi - \frac{1}{T}\right) + \frac{1}{2} \delta\left(\xi + \frac{1}{T}\right)$ & no parameter \\
   gaussian-sym  & $\text{Categorical}\left(\frac{\alpha_k}{\sum_{u=-s}^{s} \alpha_u}\right)_{k=-s}^{s}$  & truncation size $s$, gaussian std $\sigma$ \\
cosine-fold & $ \delta\left(\xi - \frac{1}{T}\right)$ & no parameter \\
   gaussian-fold & $\text{Categorical}\left(\frac{\beta_k}{\sum_{u=0}^{s} \beta_u}\right)_{k=0}^{s}$, $\beta_k = 2 \alpha_k$ when $k > 0$; $\beta_k = \alpha_k$ when $k = 0$ & truncation size $s$, gaussian std $\sigma$ \\
    \bottomrule
  \end{tabular}
\end{table*}

To capture user routines with period T, we consider the following two attention modulation kernel:  cosine prior and periodic gaussian prior.

\subsubsection{Cosine Prior.}

The simplest symmetric periodic function we can come up with is cosine with period $T$.

\begin{equation}
\begin{aligned}
 f(t) = cos(\frac{2 \pi t}{T})
\end{aligned}
\end{equation}

 Appling Corollary 3.3 yields the binary distribution for RFR frequencies:

\begin{equation}
    P(\xi) = \frac{1}{2} \delta\left(\xi - \frac{1}{T}\right) + \frac{1}{2} \delta\left(\xi + \frac{1}{T}\right)
\end{equation}
where $\delta(\cdot)$ denotes the Dirac delta function.

\subsubsection{Periodic Gaussian Prior.}

To provide more control over the temporal receptive field, we consider a  periodic gaussian prior.

\begin{equation}
\begin{aligned}
    & f(t) = e^{- \frac{1}{2}\left(\frac{t}{\sigma}\right)^2}, \quad t \in
[-\frac{T}{2}, \frac{T}{2}]\\
    & f(t) = f(t + T)
\end{aligned}
\end{equation}

A larger $\sigma$ produces a broader attention
peak, allowing interactions further apart in time to remain
correlated, while a smaller $\sigma$ enforces sharper, more localized attention.

When $\sigma \ll T/2$,  $f$ is positive
definite and we can then apply Corollary~3.3 \footnote{The periodic Gaussian
prior $f(t) = \exp(-t^2/2\sigma^2)$ is positive
definite when $\sigma \ll T/2$. Positive
definiteness requires all Fourier coefficients
$\alpha_k \geq 0$, which holds since
$\alpha_k = \frac{\sigma\sqrt{2\pi}}{T}
\exp(-2\pi^2\sigma^2k^2/T^2) \geq 0$ for all
$k \in \mathbb{Z}$. The condition $\sigma \ll T/2$
ensures the periodic extension of $f$ is smooth,
i.e. $f(\pm T/2) = \exp(-T^2/8\sigma^2) \approx 0$,
so that $f$ behaves as a proper Gaussian on each
period. For our choice
$\sigma \leq 2$\,hr and $T = 24$\,hr,
$f(\pm 12\text{hr}) = \exp(-18) \approx 10^{-8}$,
confirming positive definiteness to numerical
precision.}. The Fourier Random Rotation
matrices can be formed by sampling discrete harmonics ${\{k/T\}}_{k=-\infty}^{\infty}$ as the rotation frequencies from probability mass function $P(\xi = k/T)=\alpha_k$, where $\alpha_k$ is the Fourier coefficient of harmonic $k/T$:

\begin{align}
\label{eq:inverse}
\alpha_k & =\dfrac{1}{T} \int_{-T/2}^{T/2}
 e^{- \frac{1}{2}\left(\frac{t}{\sigma}\right)^2}\cos\left(2\pi k t/T\right) dt \notag \\
& \approx \frac{\sigma\sqrt{2\pi}}{T}
e^{-\frac{2\pi^2\sigma^2k^2}{T^2}} \\
& = \frac{1}{\sigma^* \sqrt{2\pi} } e^{- \frac{1}{2} \left(\frac{k}{\sigma^*}\right)^2 }, \quad \sigma* = \frac{T}{2 \pi \sigma} \notag
\end{align}

The approximation in \eqref{eq:inverse} comes from extending the integration limits
$[-T/2, T/2] $ to $(-\infty, \infty)$, valid
when $\sigma \ll T/2$. \footnote{For $\sigma = 2$\,hr and
$T = 24$\,hr, the boundary value
$f(\pm T/2) = \exp(-18) \approx 10^{-8}$
confirms this approximation is exact to
numerical precision.} The coefficients decay
exponentially in $k^2$,
guaranteeing rapid convergence of the truncated series and
justifying the use of a small truncation size $s \ll \infty$
in practice.

With truncation to ${\{k/T\}}_{k=-s}^{s}$, we re-normalize the probability mass function, sample:
\begin{equation}
k_0^{(\ell)}, \ldots, k_{d/2-1}^{(\ell)} \sim
\text{Categorical}\left(\frac{\alpha_k}{\sum_{u=-s}^{s} \alpha_u}\right)_{k=-s}^{s},
\quad \ell = 1, \ldots, L
\end{equation} and set $\xi_j^{(\ell)} \leftarrow
    k_j^{(\ell)} / T$ for $j = 0, \ldots, d/2-1$ in each attention layer $l$.

We compare different truncation size $s$ and variance $\sigma$ in ablation study \S\ref{sec:truncation}.

\subsubsection{Direction Awareness with Folded Frequency Distribution}
\label{sec:fold}

The frequency distributions we have derived so far from cosine prior and periodic gaussian prior are all symmetric.
In this section, we argue that sampling from folded frequency distributions maintains the modulation effect as their symmetric counterparts and has additional benefit of directional awareness. To fold a distribution, we perform 
\begin{equation}
P_{\text{fold}}(x) = 
\begin{cases} 
    2P(x) & \text{for } x > 0 \\
    P(x) & \text{for } x = 0 \\
    0 & \text{otherwise}
\end{cases}
\end{equation}
Table \ref{tab:dist} shows all the distributions we have discussed so far and their folded versions.

We first show that folded distributions maintain
the modulation effect of the symmetric ones. We have the following expression for attention logit modulation:

\begin{align}
    &(\boldsymbol{R}_m^{(j)} \boldsymbol{q}_m^{(j)})^\top
(\boldsymbol{R}_n^{(j)} \boldsymbol{k}_n^{(j)}) \\ &= (\begin{bmatrix} 
\cos(2\pi\xi_j p_m) & -\sin(2\pi\xi_j p_m) \\ 
\sin(2\pi\xi_j p_m) & \cos(2\pi\xi_j p_m) 
\end{bmatrix} \begin{bmatrix} 
q_{m,2j}  \\ 
q_{m,2j+1}
\end{bmatrix})^T  \\ &\quad \quad \begin{bmatrix} 
\cos(2\pi\xi_j p_n) & -\sin(2\pi\xi_j p_n) \\ 
\sin(2\pi\xi_j p_n) & \cos(2\pi\xi_j p_n) 
\end{bmatrix} \begin{bmatrix} 
k_{n,2j}  \\ 
k_{n,2j+1}
\end{bmatrix} \\
&= (q_{m, 2j} k_{n, 2j}  + q_{m, 2j+1} k_{n, 2j+1}) \cos(2\pi\xi_j (p_m - p_n))\\ &\quad + (q_{m, 2j} k_{n, 2j+1}  - q_{m, 2j+1} k_{n, 2j}) \sin(2\pi\xi_j (p_m - p_n))
\end{align}

Similar to RoPE, ClockRoPE is designed to modulate the upper bound of attention logits: lowering the attention logit upper bound at other times so that through softmax, more attention can be concentrated to tokens during similar hour of the day or day of the week.  Attention logits get close to the upper bound when query and key are roughly aligned.

In this case, we have  $|q_{m, 2j} k_{n, 2j+1}  - q_{m, 2j+1} k_{n, 2j}| \ll |q_{m, 2j} k_{n, 2j}  + q_{m, 2j+1} k_{n, 2j+1}|$ in (24),  the dominating cosine term does not depend on the sign of the frequencies. We can thus replace all negative frequencies with its reverse and maintain the signal modulating effect.

Sampling from folded non-negative frequencies has additional benefits of allowing the model to be directional aware and ensuring all feature pairs share a consistent rotational orientation. This consistency could help the projection matrices learn a unified representation of temporal causality.

\begin{algorithm}
\caption{ClockRoPE for 
Routine Modeling} \label{alg:clock}
\begin{algorithmic}[1]
\STATE \textbf{// Network initialization: Sample per-layer frequencies}
\FOR{each attention layer $\ell = 1, \ldots, L$}
    \STATE Set $ seed_\ell \leftarrow seed_{base} + \ell$
    \STATE Draw $d/2$ i.i.d. samples with $ seed_\ell$ :
    $$\xi_1^{(\ell)}, \ldots, \xi_{d/2}^{(\ell)} \sim 
    \text{ClockRoPE Sampling Distribution}$$
    \STATE Form rotation frequency vector $\Theta^{(\ell)} \leftarrow \left(\xi_1^{(\ell)}, \xi_2^{(\ell)}, \ldots, \xi_{d/2}^{(\ell)}\right) $
\ENDFOR
\STATE \textbf{// Forward pass: apply ClockRoPE rotation (RoPE-style efficient realization)}
\FOR{each layer $\ell = 1, \ldots, L$}
    \STATE Compute rotation angles $2\pi \Theta^{(\ell)} p_m,  2\pi \Theta^{(\ell)} p_n $ and expand them to $r^{(\ell)}_m, r^{(\ell)}_n \in \mathbb{R}^d$  by repeating each element twice.
    \STATE Define $\text{rotate\_pair}(\boldsymbol{v}) \leftarrow (-v_2, v_1, -v_4, v_3, \ldots, -v_d, v_{d-1})$
    \STATE $\boldsymbol{q}_m^{(\ell)} \leftarrow \boldsymbol{q}_m^{(\ell)} \otimes \cos\left(r^{(\ell)}_m\right) + \text{rotate\_pair}\left(\boldsymbol{q}_m^{(\ell)}\right) \otimes \sin\left(r^{(\ell)}_m\right)$
    \STATE $\boldsymbol{k}_n^{(\ell)} \leftarrow \boldsymbol{k}_n^{(\ell)} \otimes \cos\left(r^{(\ell)}_n\right) + \text{rotate\_pair}\left(\boldsymbol{k}_n^{(\ell)}\right) \otimes \sin\left(r^{(\ell)}_n\right)$
\ENDFOR
\end{algorithmic}
\end{algorithm}

\subsubsection{Combine Multiple Distributions}

To capture multiple periods e.g. daily and weekly, we can
(1) combine ClockRoPE sampling distributions with a weighted sum of respective pmf
 or (2) shard features based on the importance of each periodicity and apply corresponding sampling distribution to each feature group. For simplicity, we took the second approach in our experiments.



The ClockRoPE algorithm is summarized in Algorithm \ref{alg:clock}. In practice, we sample ClockRoPE frequencies separately for each head in multi-head attention.

\section{Experiments}

In this section, we evaluate ClockRoPE on production-scale generative
retrieval task at a major video-sharing platform. We experiment with ClockRoPE and other baseline methods on top of a transformer-style
base model, which predicts users' next interaction given their historic interaction sequences. This foundation model powers top-k retrieval from a massive video corpus on both watch page and home page, and serves as the upstream of the more fine-grained ranking model. 

We aim to answer the following questions through experiments:

\begin{enumerate}
  \item How does ClockRoPE affect attention distribution ?
  \item Does the attention distribution change benefit viewer satisfaction?
  \item How does ClockRoPE interact with RoPE ? 
\end{enumerate}

\subsection{Offline Evaluation}

\begin{table}[htbp]
  \centering
  \caption{Experiment group 1: no position embedding. Percentage MAP improvement over control.}
  \label{tab:youtube-offline1}
  \renewcommand{\arraystretch}{1.2} 

  \begin{tabular}{lrr} 
    \toprule
    \textbf{Method} & \textbf{MAP@1} & \textbf{MAP@50} \\
    \midrule
   Arm 1 (Fourier Features)      & +0.43\% & +0.29\% \\
   Arm 2 (ClockRoPE-cosine-sym)   & +1.89\% & +1.25\% \\
Arm 3 (ClockRoPE-cosine-fold) & +1.96\% & +1.37\% \\
    Arm 4 (ClockRoPE-gaussian-sym) & +3.19\% & +2.12\% \\
    Arm 5 (ClockRoPE-gaussian-fold) & +3.61\% & +2.25\% \\
    \bottomrule
  \end{tabular}
\end{table}

\begin{table}[htbp]
  \centering
  \caption{Experiment group 2: combining with RoPE. Percentage MAP improvement over RoPE only control.}
  \label{tab:youtube-offline2}
  \renewcommand{\arraystretch}{1.2} 

  \begin{tabular}{lrr} 
    \toprule
    \textbf{Method} & \textbf{MAP@1} & \textbf{MAP@50} \\
    \midrule
    Arm 1 (ClockRoPE-cosine-fold) & +1.02\% & +1.17\% \\
   Arm 2 (ClockRoPE-gaussian-sym)  & +2.18\% & +1.91\% \\
    Arm 3 (ClockRoPE-gaussian-fold) & +2.39\% & +2.00\% \\
    \bottomrule
  \end{tabular}
\end{table}

We conducted two groups of experiments to investigate ClockRoPE's effectiveness with and without RoPE. 

\subsubsection{Experiment Group 1 }
All the variants in this group use absolute timestamp embedding and no position embedding. For ClockRoPE-gaussian methods, we used parameter sweep to select truncation size and variance. 
\begin{itemize}
  \item \textbf{Control:}  Hour-of-the-day and day-of-the-week input features : $ [ t \bmod T_d, \quad \left\lfloor \frac{t}{T_d} \right\rfloor \bmod \frac{T_w}{T_d} ] $ . This is the production baseline.
  \item \textbf{Arm 1 (Fourier Features):}  Daily and weekly Fourier input features\\ $[\cos(\frac{2\pi t}{T_d}), \sin(\frac{2\pi t}{T_d}), \cos(\frac{2\pi t}{T_w}), \sin(\frac{2\pi t}{T_w})]$. This periodic encoding has the nice property of translation invariance as the cosine difference formula ensures that dot product between two such feature vectors depends only on the relative temporal distance. 
   \item \textbf{Arm 2 (ClockRoPE-cosine-sym):} We divide all features into two halves and apply ClockRoPE-cosine-sym with daily period and weekly period to each half respectively.
    \item \textbf{Arm 3 (ClockRoPE-cosine-fold):} Same as Arm 2 except for using cosine-fold sampling distribution.
   \item \textbf{Arm 4 (ClockRoPE-gaussian-sym):} Same as Arm 2 except for using gaussian-sym sampling distribution.
   \item \textbf{Arm 5 (ClockRoPE-gaussian-fold):} Same as Arm 2 except for using gaussian-fold sampling distribution.

\end{itemize}

As shown in Table \ref{tab:youtube-offline1}, ClockRoPE-gaussian variants consistently outperform ClockRoPE-cosine variants. A potential explanation
lies in the range of the two priors: the gaussian prior is
non-negative and decays smoothly toward $0$ as temporal distance $t \bmod T$
approaches $T/2$. The cosine prior, in
contrast, ranges over $[-1, 1]$ and reaches $-1$ at $t = T/2$,
actively inverting the sign of the attention logit for maximally
out-of-phase pairs rather than merely suppressing it. This sign
inversion can inject additional noise into the attention
distribution.

 We selected the best performing variants from group 1 and experimented with combining them with RoPE in the next section.

\subsubsection{Experiment Group 2 }
\label{sec:group2}
In the second group, we evaluated how ClockRoPE interacts with RoPE.  All the variants in this group use both absolute timestamp embedding and standard RoPE for  position embedding. We combined ClockRoPE with RoPE by applying RoPE to half of the hidden features and ClockRoPE to the other half. Other ways of combining RoPE and ClockRoPE are discussed in Ablation \S\ref{sec:truncation}.

\begin{itemize}
  \item \textbf{Control:} Hour-of-the-day and day-of-the-week input features : $ [ t \bmod T_d, \quad \left\lfloor \frac{t}{T_d} \right\rfloor \bmod \frac{T_w}{T_d} ] $ with standard RoPE .
  \item \textbf{Arm 1 (ClockRoPE-cosine-fold):}  We divide all features
into two halves and apply RoPE to the first half. We further divide the second half evenly into two feature groups and apply ClockRoPE-cosine-fold with
daily period and weekly period to each group respectively.
   \item \textbf{Arm 2 (ClockRoPE-gaussian-sym):} Same as Arm 1 except for using gaussian-sym sampling distribution.
   \item \textbf{Arm 3 (ClockRoPE-gaussian-fold):} Same as Arm 1 except for using gaussian-fold sampling distribution.

\end{itemize}

Table \ref{tab:youtube-offline1}  and 
Table \ref{tab:youtube-offline2} show that ClockRoPE-gaussian-fold performs the best both with and without RoPE.

\subsection{Online A/B Tests}
\label{sec:online}

\begin{table}[htbp]
  \centering
  \caption{RoPE and ClockRoPE each independently improves sitewide valued engagement. The best result is achieved by combining RoPE and ClockRoPE.}
  \label{tab:youtube-eval}
  \renewcommand{\arraystretch}{1.2} 

  \begin{tabular}{lrr} 
    \toprule
    \textbf{Method} & \textbf{Engagement} & \textbf{Valued Engagement} \\
    \midrule
    RoPE Only        & +0.04\% & +0.06\% \\
    ClockRoPE Only   & +0.03\% & +0.05\% \\
    RoPE + ClockRoPE & +0.08\% & +0.08\% \\
    \bottomrule
  \end{tabular}
\end{table}

We conducted an online A/B test for 14 days, assigning $1\%$ total traffic to each variant. In this study, we used the best config from offline evaluations:  ClockRoPE-gaussian-fold, where ClockRoPE and RoPE are combined in the
same way as in \S\ref{sec:group2}, each applying to half of the features.

As shown in Table \ref{tab:youtube-eval},  RoPE and ClockRoPE are complementary to each other: 

\begin{itemize}
  \item Applying RoPE or ClockRoPE alone each helps improve sitewide valued engagement.
  \item  The best result is achieved by combining RoPE and ClockRoPE, leveraging both long-term attention decay of RoPE and periodicity-awareness of ClockRoPE. 

\end{itemize}

More analysis on how ClockRoPE interacts with RoPE in terms of attention distribution can be found in \S\ref{sec:attn}.

\subsection{Ablation Studies}

\subsubsection{Truncation Size and Variance in ClockRoPE Gaussian}
\label{sec:truncation}
We ablated ClockRoPE-gaussian-fold with single period $T=24$ hr on both truncation size $s$ and variance. The results are shown in Table \ref{tab:truncation}.  We see that tuning variance has a larger impact than truncation size. When the truncation size is large enough e.g. 6 in our case, keep increasing it does not yield extra performance gain. 

\begin{table}[htbp]
  \centering
  \caption{Percentage MAP improvement over prod control with varying truncation size and variance.}
  \label{tab:truncation}
  \renewcommand{\arraystretch}{1.2} 
  \begin{tabular}{lrr} 
    \toprule
    \textbf{Method} & \textbf{MAP@1} & \textbf{MAP@50} \\
    \midrule
 $s = 6$, $\sigma = 1.0$ hr & +0.95\% & +0.65\% \\
 $s = 6$, $\sigma = 1.5$ hr & +1.51\% & +0.85\% \\
\bfseries $\boldsymbol{s = 6}$, $\boldsymbol{\sigma = 2.0}$ \bfseries hr & \bfseries +2.12\% & \bfseries +2.05\% \\
 $s = 6$, $\sigma = 2.5$ hr & +1.95\% & +1.84\% \\
 $s = 3$, $\sigma = 2.0$ hr & +1.59\% & +1.12\% \\
 $s = 9$, $\sigma = 2.0$ hr & +1.88\% & +1.41\% \\
 $s = 12$, $\sigma = 2.0$ hr & +2.01\% & +1.94\% \\
    \bottomrule
  \end{tabular}
\end{table}

\subsubsection{Different Ways of Combining RoPE and ClockRoPE}
\label{sec:combine}
We compared two approaches:

\begin{itemize}
    \item Feature-wise division: rotate half hidden features with RoPE and the other half with ClockRoPE.
    \item  Head-wise division: rotate all features in half heads with RoPE and features in the other half  heads with ClockRoPE.
\end{itemize}

Table \ref{tab:combine} shows that combining RoPE and ClockRoPE feature-wise performs slightly better in offline evaluation.

\begin{table}[htbp]
  \centering
  \caption{Percentage MAP improvement over RoPE only control. Feature-wise division performs slightly better than Head-wise division.}
  \label{tab:combine}
  \renewcommand{\arraystretch}{1.2} 

  \begin{tabular}{lrr} 
    \toprule
    \textbf{Method} & \textbf{MAP@1} & \textbf{MAP@50} \\
    \midrule
   Feature-wise division  & +2.39\% & +2.00\% \\
Head-wise division & +1.94\% & +1.87\% \\
    \bottomrule
  \end{tabular}
\end{table}

\subsection{Visualization of Temporal Attention Distribution}

\subsubsection{How ClockRoPE Interacts with RoPE}
\label{sec:attn}

\begin{figure}
    \centering
\includegraphics[width=1.0\linewidth]{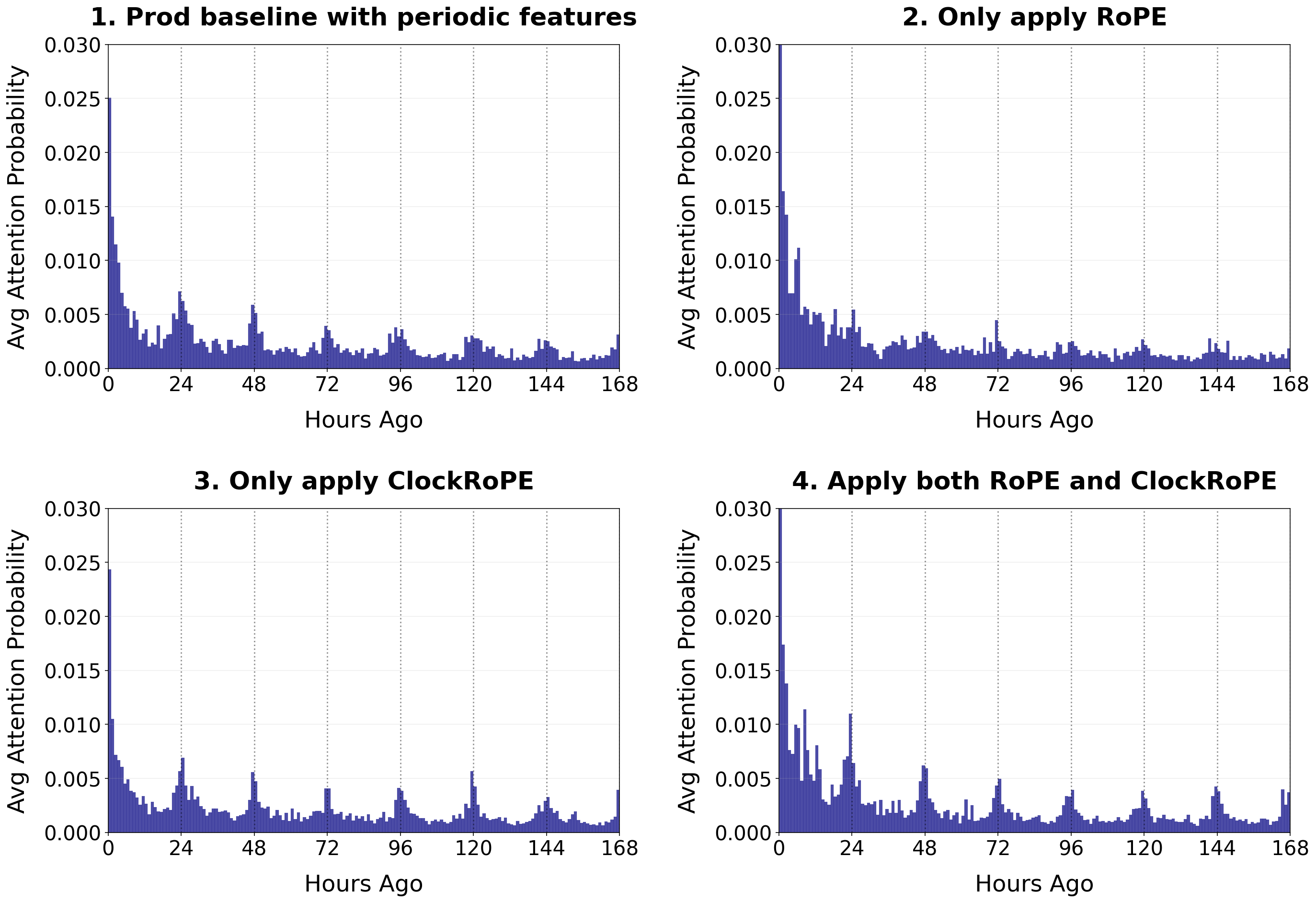}
    \caption{Visualization of average attention probability over a week. Key observations: (2) Standard RoPE allocates more attention to the most recent history  and obscures daily period. (3) Applying ClockRoPE alone sharpens attention peak around every 24 hrs. (4) The combination of RoPE and ClockRoPE retains both long-term decay and  clear daily period.   }
    \label{fig:attn}
\end{figure}

Figure \ref{fig:attn} shows the attention distribution over temporal distances ranging from 0 to 168hr, averaging attention probabilities for every hour bucket. In this comparison, we used ClockRoPE-gaussian-fold with daily period only.

Figure \ref{fig:attn} subgraph 2 confirms the motivation in \S\ref{sec:motivation}: the log-linear frequency schedule's long-term decay dominates the attention pattern, suppressing the periodic signal that ClockRoPE is designed to recover. Subgraph 4 visually verifies that both  long-term decay and clear daily period are retained with the combination of RoPE and ClockRoPE, which explains why the combination achieves the best online performance gain in \S\ref{sec:online}.

We do not highlight the difference between Gaussian-symmetric and Gaussian-fold here as they produce similar pattern at the hourly level.

\subsubsection{Symmetric vs. Folded ClockRoPE Frequency Distribution} 
\begin{figure}
    \centering
    \includegraphics[width=1.0\linewidth]{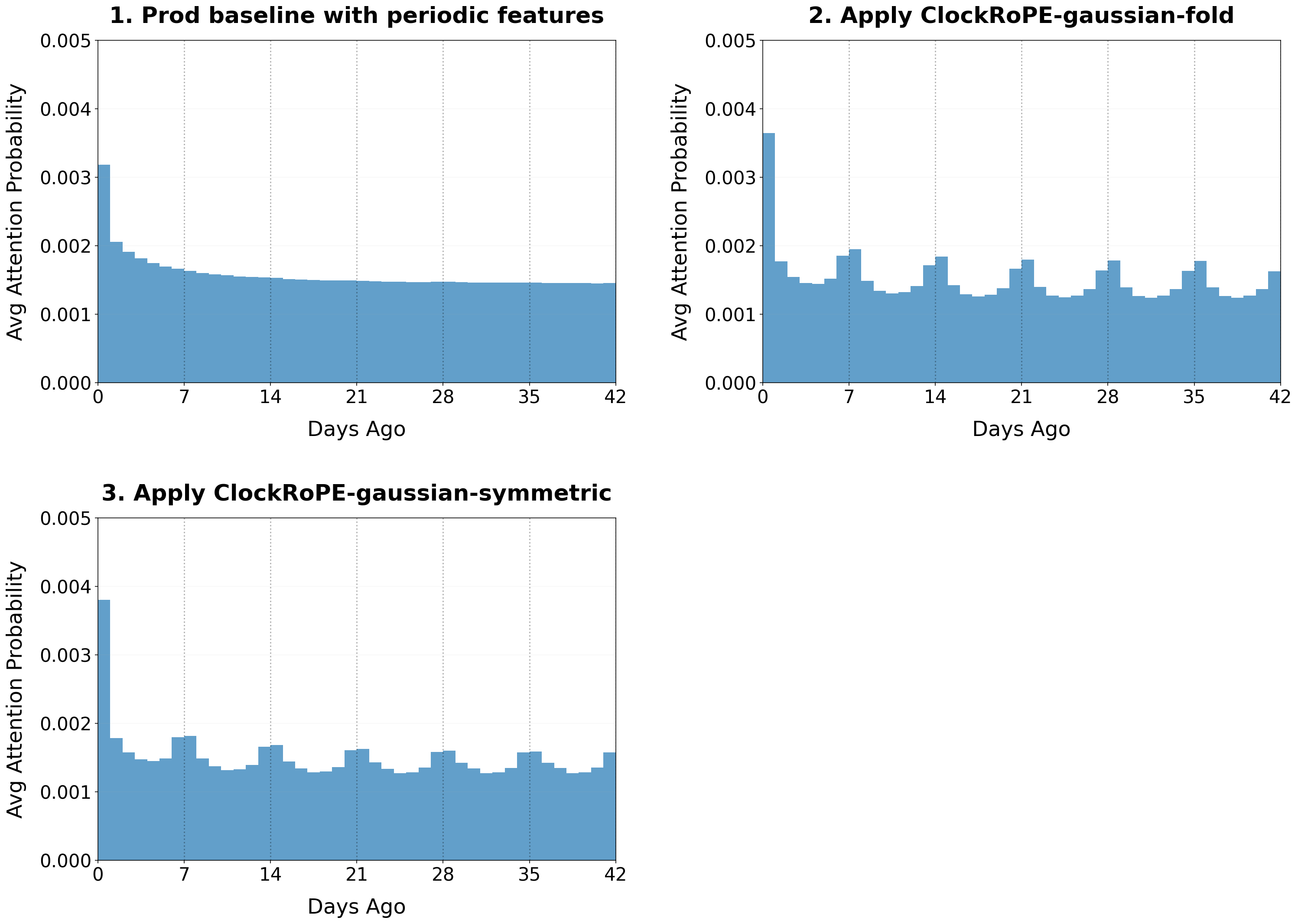}
    \caption{Visualization of average attention probability over six weeks. Key observations: (1) ClockRoPE introduces weekly attention peak that is not explicit in the baseline.  (2) With ClockRoPE-gaussian-fold, more attention is allocated to the day immediate before the weekly cycle mark than the day immediate after (3) The model with ClockRoPE-gaussian-symmetric allocates attention much more evenly between the two days surrounding the weekly cycle mark.}
    \label{fig:week}
\end{figure}

When averaging attention probabilities for every daily bucket as in Figure \ref{fig:week}, we start to observe difference between folded frequency distribution and its symmetric counterpart. As discussed in \S\ref{sec:fold}, the folded frequency distribution allows the model to distinguish between interactions one day before the weekly period mark and interactions one day after by performing ClockRoPE rotations in different directions. In contrast, interactions before and  after the weekly period mark result in the same rotations with ClockRoPE-gaussian-symmetric as long as their temporal distance to the weekly period mark is the same.  Figure \ref{fig:week} subgraph 2 suggests that with the advantage of being directional-aware, the model learns to allocate more attention to interactions before it than those after it when projecting all interactions onto the same weekly circle e.g. an interaction on Monday distributes more attention to interactions on the past Sunday (8 days ago) than those on the past Tuesday (6 days ago). This explains that a potential reason for the performance improvement of ClockRoPE-gaussian-fold over its symmetric counterpart is its capability of capturing routines with causal relationships.

\subsection{Production Deployment Results}

\begin{table}[htbp]
  \centering
  \caption{Core Deployment Metrics}
  \label{tab:deploy}
  \renewcommand{\arraystretch}{1.2} 
  \begin{tabular}{lr}
    \toprule
    \textbf{Core Deployment Metrics} & \textbf{Results} \\
    \midrule
   Valued Engagement Time & +0.08\% \\
   Homepage Triggered Engagement Time & +0.11\% \\
    Homepage Views & -0.51\% \\
     Daily Unique Impressed Videos & -0.19\% \\
     TPU Serving Cost & -0.63\% \\
   Serving Latency
 & Neutral \\
    \bottomrule
  \end{tabular}
\end{table}

\begin{table}[htbp]
  \centering
  \caption{Routine-related Metrics }
  \label{tab:online-results}
  \renewcommand{\arraystretch}{1.2} 
  \begin{tabular}{lr}
    \toprule
    \textbf{Routine-related Metrics} & \textbf{Online Lift} \\
    \midrule
   Engaged Vod Topics & +0.05\% \\
    Repeat Engaged Topics & +0.03\% \\
     Gaming-Related Content Valued Engagement & +0.08\% \\
     News Valued
 Engagement & +0.17\% \\
    Learning Content Valued Engagement
 & +0.15\% \\
    \bottomrule
  \end{tabular}
\end{table}

ClockRoPE has been successfully
deployed together with RoPE across homepage and watch page  at a major
video-sharing platform with neutral serving latency impact. As shown in Table \ref{tab:deploy}, we observe fewer daily unique impressed videos and homepage views, alongside more homepage triggered engagement time and overall valued engagement time, suggesting that it now takes users fewer scrolls and clicks to find satisfying contents, which in turn reduces serving volume and saves overall TPU serving cost by $0.63 \%$.

 In addition to the core deployment metrics, Table \ref{tab:online-results} demonstrates evidence that ClockRoPE helps improve sitewide valued engagement of routine-related watch behavior, such as learning, gaming and news watching,  tying back to our initial motivation of improving routine-awareness in sequential recommendation.

\section{Conclusion}

In this work, we introduced ClockRoPE, a mathematically grounded periodic encoding for modeling user routines. By leveraging the framework of Random Fourier Rotations, ClockRoPE achieves performance gain without requiring additional input features or increasing model dimensionality. Crucially, ClockRoPE serves as a powerful complement to standard RoPE, enabling transformers to capture natural behavioral periodicity alongside long-term attention decay. The successful deployment of ClockRoPE in a major video-sharing platform, yielding consistent valued engagement gains with resource savings, confirms its effectiveness and efficiency for large-scale industrial applications.

\clearpage

\bibliographystyle{ACM-Reference-Format}
\bibliography{main}

\appendix

\section{Proofs}

\subsection{Proof of Proposition 3.2}

\begin{proof}
Define the real-valued function:
\begin{align}
h(\xi_0, \ldots, \xi_{d/2-1}) &= \sum_{j=0}^{d/2-1} 
(\boldsymbol{R}_m^{(j)} \boldsymbol{q}_m^{(j)})^\top 
(\boldsymbol{R}_n^{(j)} \boldsymbol{k}_n^{(j)}) \notag \\
&= \sum_{j=0}^{d/2-1} \mathrm{Re}\left[\boldsymbol{q}_m^{(j)} 
{\boldsymbol{k}_n^{(j)}}^* e^{i2\pi\xi_j(p_m - p_n)}\right]
\end{align}

By Proposition 3.1, $\mathbb{E}[h] = \boldsymbol{q}_m^\top 
\boldsymbol{k}_n f(p_m - p_n)$.Changing $\xi_j$ to $\xi_j'$ while keeping all other 
frequencies fixed, the change in $h$ is bounded by:

\begin{align}
& |h(\ldots, \xi_j, \ldots) - 
h(\ldots, \xi_j', \ldots)| \notag \\
&=  \left|\mathrm{Re}\left[
\boldsymbol{q}_m^{(j)}{\boldsymbol{k}_n^{(j)}}^*
e^{i2\pi\xi_j(p_m-p_n)}\right] - \mathrm{Re}\left[
\boldsymbol{q}_m^{(j)}{\boldsymbol{k}_n^{(j)}}^*
e^{i2\pi\xi_j'(p_m-p_n)}\right]\right| \notag \\
&\leq \left|\boldsymbol{q}_m^{(j)}
{\boldsymbol{k}_n^{(j)}}^*\right| \cdot 
\left| e^{i2\pi\xi_j(p_m-p_n)} - e^{i2\pi\xi_j'(p_m-p_n)} \right| \notag \\
&\leq 2\left|\boldsymbol{q}_m^{(j)}{\boldsymbol{k}_n^{(j)}}^*
\right| \notag \\
&= 2\left|\boldsymbol{q}_m^{(j)}\right|\left|{\boldsymbol{k}_n^{(j)}}^*
\right| \notag \\
&= 2\left\|\boldsymbol{q}_m^{(j)}\right\| \left\|\boldsymbol{k}_n^{(j)}\right\|
\end{align}

 The last inequality follows from  the maximum distance between any two points on 
the unit circle is 2. The last equality comes from replacing complex norm with corresponding 2d vector norm.

Denote bound $c_j = 2\left\|\boldsymbol{q}_m^{(j)}\right\| \left\|\boldsymbol{k}_n^{(j)}\right\|$ and apply McDiarmid's inequality to $h$:

\begin{align}
&P\left(\left|\sum_{j=0}^{d/2-1}(\boldsymbol{R}_m^{(j)} 
\boldsymbol{q}_m^{(j)})^\top(\boldsymbol{R}_n^{(j)} 
\boldsymbol{k}_n^{(j)}) - \boldsymbol{q}_m^\top\boldsymbol{k}_n 
f(p_m-p_n)\right| \geq \epsilon\right) \notag \\
&\leq 2\exp\left(-\frac{2\epsilon^2}{\sum_{j=0}^{d/2-1} 
c_j^2}\right) \notag \\
&= 2\exp\left(-\frac{\epsilon^2}{2\sum_{j=0}^{d/2-1}
(\left\|\boldsymbol{q}_m^{(j)}\right\| \left\|\boldsymbol{k}_n^{(j)}\right\|)^2}
\right)
\end{align}

replacing $\epsilon$ with $\frac{d}{2}\epsilon'$ completes the proof.
\end{proof}
\subsection{Proof of RFR Exponential Convergence with Feature Dimension at Initialization}

 Let $X_j = \left\|\boldsymbol{q}_m^{(j)}\right\| \left\|\boldsymbol{k}_n^{(j)}\right\|$. When query and key projection matrices are initialized i.i.d. with variance independent of d, we have $\{X_j\}_{j=0}^{d/2-1}$ i.i.d. across $j$ at initialization and  $\mathbb{E}[X_j^2]$ remains constant as d scales. Applying the law of large numbers to $X_j^2$, we obtain
\[
\frac{1}{d/2}\sum_{j=0}^{d/2-1} X_j^2 \to \mathbb{E}[X_j^2],
\]
which implies $\sum_{j=0}^{d/2-1} X_j^2 = \Theta(d)$.
In Proposition 3.2, substituting $\sum_{j=0}^{d/2-1}
(\left\|\boldsymbol{q}_m^{(j)}\right\| \left\|\boldsymbol{k}_n^{(j)}\right\|)^2$ with $\Theta(d)$ yields that \[
\frac{1}{d/2}\sum_{j=0}^{d/2-1}
(\boldsymbol{R}_m^{(j)}\boldsymbol{q}_m^{(j)})^\top
(\boldsymbol{R}_n^{(j)}\boldsymbol{k}_n^{(j)})
\]
converges to
\[
\frac{1}{d/2}\boldsymbol{q}_m^\top\boldsymbol{k}_n
f(p_m-p_n)
\] exponentially fast in $d$.

\section{Can Rotation Frequencies Be Learnt?}
We conducted experiments that jointly optimize rotation frequencies with network weights. In Table \ref{tab:learn}, we compare learnable frequencies with 0 initialization and ClockRoPE initializations. We also list ClockRoPE variants with fixed frequencies for reference. We find that (1) using ClockRoPE initialization improves performance over 0 initialization and (2) trainable frequencies initialized with ClockRoPE yields similar performance as fixed frequencies.

\begin{table}[ht]
\centering
\footnotesize 
\caption{Percentage MAP improvement over prod control. Learnable frequncies initialized with ClockRoPE achieve similar performance as fixed frequencies.}
\label{tab:learn}
\setlength{\tabcolsep}{4pt} 
\begin{tabular}{lccc} 
\toprule
\textbf{Method} & \textbf{MAP@1} & \textbf{MAP@50}  \\
\midrule
Learnable frequencies initialized from 0.0           &  -1.44\% & -1.40\% \\

\midrule
ClockRoPE-cosine-fold         &
 +1.96\% & +1.37\% \\
\quad + Learnable frequencies &  +2.01\% & +1.45\% \\
\midrule
ClockRoPE-gaussian-fold            & +3.61\% & +2.25\% \\
\quad + Learnable frequencies &  +3.23\% & +2.16\% \\
\bottomrule
\end{tabular}
\end{table}

\end{document}